\documentclass[letterpaper]{article}
\usepackage[preprint]{aaai2027}
\usepackage[hyphens]{url}  % DO NOT CHANGE THIS
\usepackage{graphicx} % DO NOT CHANGE THIS
\usepackage{natbib}  % DO NOT CHANGE THIS AND DO NOT ADD ANY OPTIONS TO IT
\usepackage{caption} % DO NOT CHANGE THIS AND DO NOT ADD ANY OPTIONS TO IT
\usepackage{algorithm}
\usepackage{algorithmic}
\usepackage{booktabs}
\usepackage{amsmath,amssymb,amsfonts}

\newtheorem{definition}{Definition}
\newtheorem{theorem}{Theorem}

\title{COVENANT: Natural-Language Workflow Compilation for Aligned Agent Execution}

\author{
    Jincheng Wang, Min Zheng, Tao Wei
}
\affiliations{
    Ant Group\\
    \texttt{wjcuhk@gmail.com, zhengmin.zm@antgroup.com, lenx.wei@antgroup.com}
}

\begin{document}

\maketitle

\begin{abstract}
Large language model (LLM) agents are increasingly entrusted with natural-language workflow instructions (e.g., retail-payment policies) that specify not only what outcome to achieve, but also which steps, branches, and tool interactions are permitted. When these instructions are supplied as prompt context, however, the model retains control over both procedure selection and step execution. As interactions accumulate, an agent can skip required steps, take unsupported branches, or execute a valid step with unsupported arguments or effects---a failure mode we call \emph{workflow misalignment}. In this work, we propose COVENANT, a compiler-and-interpreter architecture for workflow-aligned agent execution. Our key insight is to treat workflow instructions as source programs rather than prompts. COVENANT converts the instructions into a workflow abstract syntax tree (WAST) and lowers it to a workflow control-flow graph (WCFG). At runtime, a controller interprets the WCFG one node at a time, checks each proposal against requirements extracted from the instructions before committing controller state or advancing the graph, and returns diagnostic feedback for repair. To evaluate COVENANT, we use 120 cases from three existing benchmarks, spanning seven workflow scenarios. Compared with state-of-the-art LLM agents, COVENANT improves benchmark success from 50.00\% to 83.33\% and reduces the workflow-misalignment failure rate from 42.50\% to 15.83\% (62.75\% relative). These results show that COVENANT substantially mitigates workflow misalignment, moving LLM-agent alignment beyond isolated prompt following toward reliable execution of complex and multi-step workflows.
\end{abstract}

\section{Introduction}
\label{sec:introduction}

Large language model (LLM) agents are increasingly used to handle complex tasks that require multi-step reasoning, tool use, and interaction with external environments. Many such tasks are governed by workflow instructions written in natural language, including standard operating procedures (SOPs), service policies, compliance procedures, and tool-use protocols~\cite{guidebench2025,lu2024toolsandbox,yao2024taubench}. For example, a weather-assistance workflow may require reusing a location already supplied by the user, while an airline-booking workflow may require retrieving a missing passenger birth date through the profile tool before booking. The responsibility for following such instructions step by step, traditionally assigned to human operators, is now increasingly delegated to LLM agents. We call this setting \emph{natural-language workflow following}. In practice, agents following complex natural-language workflows frequently drift from the prescribed execution by skipping required steps, choosing unsupported branches, or incorrectly performing required tool interactions, leading to incorrect outcomes or premature workflow termination. Such deviations constitute workflow misalignment. Figure~\ref{fig:motivating} illustrates two representative examples.

\begin{figure*}[!t]
\centering
\includegraphics[width=0.85\textwidth]{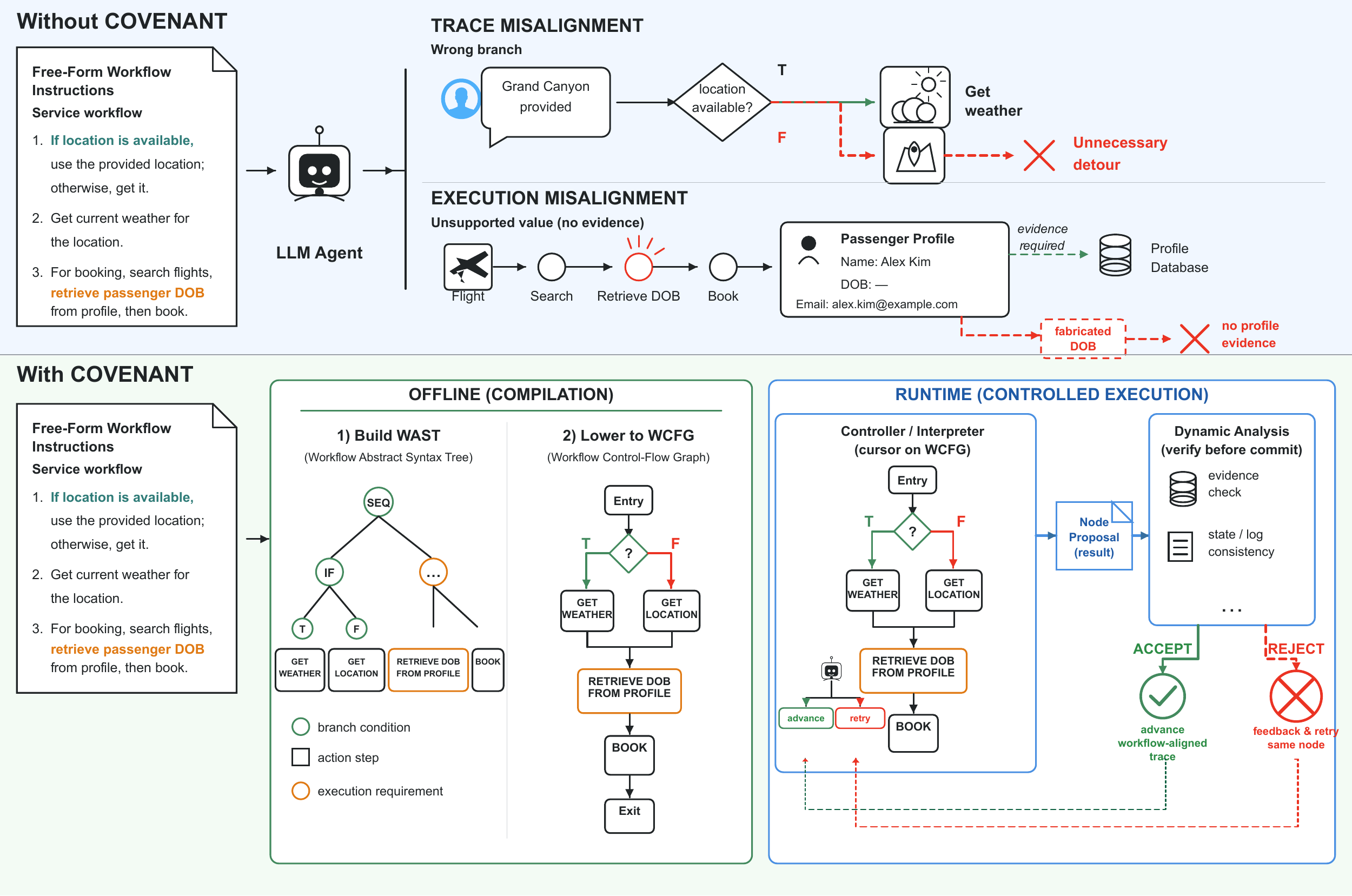}
\caption{Problem-to-solution overview of COVENANT. An LLM agent can select an unsupported branch (trace misalignment) or execute a selected step without required evidence (execution misalignment). COVENANT compiles natural-language workflow instructions into a WAST and WCFG, then combines controller-guided traversal with runtime verification to support workflow-aligned execution.}
\label{fig:motivating}
\end{figure*}

Recent systems such as AgentLTL and FlowAgent have begun to improve procedural compliance and controlled workflow execution~\cite{agentltl2026,flowagent2025}. However, they generally assume that workflow instructions have already been rewritten into formal or semi-structured representations with prescribed syntax. This assumption is difficult to satisfy in practice. Real-world workflow instructions are commonly written for people in flexible document formats and unrestricted natural language; requiring users to translate them into prescribed representations introduces substantial authoring and maintenance costs, limiting the direct applicability of existing methods. This work therefore asks: \emph{Given workflow instructions written in natural language, how can we ensure that an LLM agent's execution trace remains aligned with them?}

To address this issue, our key observation is that workflow instructions, although written in free-form natural language, usually encode program-like procedural structure, including ordered steps, branch conditions, and execution requirements. Inspired by traditional compilation~\cite{aho2006compilers,lattner2004llvm}, we therefore treat workflow instructions as source programs and propose \emph{COVENANT}, a compiler-and-interpreter architecture for workflow-aligned agent execution. As Figure~\ref{fig:motivating} shows, an offline compiler first processes raw workflow instructions to construct a workflow abstract syntax tree (WAST), then lowers the WAST into a workflow control-flow graph (WCFG). At runtime, a controller interprets the WCFG one node at a time, asks the target LLM agent to execute the active node, and checks the result against requirements extracted from the instructions. The controller advances only after the result passes; otherwise, it returns diagnostic feedback for repair. Controller-owned traversal targets trace misalignment, while node-level checking targets execution misalignment.

We evaluate COVENANT across seven real-world workflow scenarios. Example scenarios include multi-turn personal-assistant tool use, airline reservation, and retail order management, covering distinct practical domains and workflow structures. Results show that COVENANT consistently improves workflow alignment across diverse agent systems and underlying LLMs. In its best evaluated configuration, Hermes Agent with GLM-5.1~\cite{nousresearch2026hermesagent,glm5team2026glm5}, COVENANT raises benchmark success from 50.00\% to 83.33\% and reduces the workflow-misalignment failure rate from 42.50\% to 15.83\% (62.75\% relative reduction).

In summary, our contributions are:
\begin{itemize}
\item To the best of our knowledge, COVENANT is the first system to address natural-language workflow misalignment issues. We formulate natural-language workflow alignment as an agent--environment trace-completeness problem (\S\ref{sec:problem-definition}).
\item We introduce COVENANT, which adapts compilation theory to natural-language workflows, then uses a controller to interpret the compiled artifact and supervise execution.
\item We provide a systematic evaluation over diverse agent systems and base LLMs in seven workflow scenarios. The results show that COVENANT generalizes across real-world workflow instructions written in free-form natural language, reducing the workflow-misalignment failure rate by 62.75\% (\S\ref{sec:evaluation}).
\end{itemize}

\section{Problem Definition}
\label{sec:problem-definition}

\subsection{Workflow Execution}
\label{sec:workflow-execution}

We model workflow execution as an interaction among an agent, a workflow specification, and an environment. Let $\mathcal{S}$ be the set of agent states, $\mathcal{A}$ the set of actions, $\mathcal{O}$ the set of feedback observations produced by the environment, and $\mathcal{W}$ the set of workflow specifications. A \emph{workflow specification} $W\in\mathcal{W}$ is a procedural instruction document that specifies how a task should be carried out. It may prescribe ordered steps, branch conditions, action constraints, and completion requirements. We treat the task input $I\in\mathcal{O}$ as the initial observation $o_0=I$. The agent's policy is
\begin{equation}
\pi:\mathcal{S}\times\mathcal{W}\times\mathcal{O}\rightarrow\mathcal{A},
\qquad
a'_t=\pi(s_t,W,o_t),
\label{eq:agent-policy}
\end{equation}
where $s_t\in\mathcal{S}$ is the agent state, $o_t\in\mathcal{O}$ is the latest environment feedback, and $a'_t\in\mathcal{A}$ is the action selected by the agent at step $t$. In particular, execution begins with $a'_0=\pi(s_0,W,I)$.

The selected action changes the environment, which returns the next feedback. Let $h_t=\langle I,a'_0,o_1,\ldots,a'_{t-1},o_t\rangle$ be the interaction history before step $t$, $\mathsf{Env}$ the environment feedback function, and $U$ the agent-state update function. The interaction evolves as
\begin{equation}
o_{t+1}=\mathsf{Env}(h_t,a'_t),
\qquad
s_{t+1}=U(s_t,a'_t,o_{t+1}).
\label{eq:agent-environment-step}
\end{equation}
Thus, $W$ remains fixed as procedural guidance, while the policy and environment jointly determine the actions and feedback observed during execution.

For a workflow specification $W$ and task input $I$, let $\mathcal{T}(W,I)$ denote the set of interaction traces permitted by the workflow. A trace records the task input, the ordered actions, and the feedback returned by the environment. Any $\tau(W,I)\in\mathcal{T}(W,I)$ is therefore an intended workflow execution. The policy--environment interaction produces the realized trace
\begin{equation}
\tau_{\pi,\mathsf{Env}}(W,I)
=\langle I,a'_0,o_1,a'_1,o_2,\ldots,a'_{m-1},o_m\rangle.
\label{eq:realized-workflow-trace}
\end{equation}
Because a workflow may contain branches or allow equivalent actions, $\mathcal{T}(W,I)$ may contain more than one valid trace.

Specifically, in this work, we define workflow alignment as follows.

\begin{definition}[Workflow Alignment]
\label{def:workflow-alignment}
Given $W$ and $I$, the execution induced by policy $\pi$ and environment $\mathsf{Env}$ is \emph{workflow-aligned} if and only if
\begin{equation}
\tau_{\pi,\mathsf{Env}}(W,I)
\in \mathcal{T}(W,I).
\label{eq:workflow-alignment}
\end{equation}
Otherwise, the execution is \emph{workflow-misaligned}.
\end{definition}

In this work, we assume that the environment $\mathsf{Env}$ is benign and trusted: given an action and its arguments, it faithfully executes the requested operation and returns the resulting feedback. Under this assumption, violating the alignment requirement in Equation~\ref{eq:workflow-alignment} indicates a failure of policy $\pi$. We refer to a change in the permitted procedural path as \emph{trace misalignment} and an incorrect realization of a selected step as \emph{execution misalignment}; Appendix~\ref{app:workflow-misalignment} provides the detailed characterization and controlled stress analysis.

\section{Workflow Compilation}
\label{sec:workflow-compilation}

Workflow instructions encode program-like procedures, but asking an LLM to reinterpret and orchestrate the complete procedure at runtime permits execution drift. Inspired by traditional compilation~\cite{aho2006compilers,lattner2004llvm}, COVENANT instead treats the instructions as a source program. Its offline compiler constructs a \emph{workflow abstract syntax tree} (WAST) that makes actions, guards, and their hierarchy explicit, then lowers it to a \emph{workflow control-flow graph} (WCFG). Executable and decision nodes represent workflow steps and guards, while directed edges encode permitted transitions. The resulting WCFG is analogous to graph-structured workflow bytecode interpreted by the runtime controller (Figure~\ref{fig:motivating}).

\subsection{Offline Compilation}
\label{subsec:offline-compilation}

Offline compilation first reconstructs the instruction hierarchy as a WAST and then lowers it to a WCFG. Figure~\ref{fig:wast-wcfg-syntax} defines both artifacts: the WAST organizes inputs, resources, outputs, steps, and branches, while the WCFG represents the procedure as typed nodes connected by guarded edges with explicit state access.

\begin{figure}[t]
\centering
\begingroup
\setlength{\tabcolsep}{2pt}
\renewcommand{\arraystretch}{1.12}
\hrule
\vspace{4pt}
\centering
\textbf{Workflow Abstract Syntax Tree (WAST)}\\[3pt]
\footnotesize
\begin{tabular}{@{}r@{ }c@{ }l@{}}
$\langle\mathit{WAST}\rangle$ & $::=$ & \texttt{workflow} $\langle Id\rangle\;\langle Input\rangle^{*}\;\langle Resource\rangle^{*}$ \\
& & $\langle Tree\rangle\;\langle Output\rangle^{*}$ \\
$\langle Tree\rangle$ & $::=$ & \texttt{seq} $\langle Tree\rangle^{+}$ $\mid$ \texttt{step} $\langle Elem\rangle^{+}$ \\
& $\mid$ & \texttt{branch} $\langle Guard\rangle\;\langle Tree\rangle\;\langle Tree\rangle^{?}$ \\
$\langle Guard\rangle$ & $::=$ & \texttt{guard} $\langle Pred\rangle$ \\
$\langle Elem\rangle$ & $::=$ & \texttt{action} $\langle Task\rangle$ $\mid$ \texttt{output} $\langle Slot\rangle$
\end{tabular}
\par\medskip
\textbf{Workflow Control-Flow Graph (WCFG)}\\[3pt]
\begin{tabular}{@{}r@{ }c@{ }l@{}}
$\langle\mathit{WCFG}\rangle$ & $::=$ & \texttt{graph} $\langle State\rangle^{*}\;\langle Node\rangle^{+}\;\langle Edge\rangle^{+}$ \\
$\langle State\rangle$ & $::=$ & \texttt{slot} $\langle Id\rangle$ \texttt{:} $\langle Type\rangle$ \\
$\langle Node\rangle$ & $::=$ & \texttt{entry} $\langle Id\rangle$ $\mid$ \texttt{merge} $\langle Id\rangle$ \\
& $\mid$ & \texttt{block} $\langle Id\rangle$ \texttt{task} $\langle Task\rangle\;\langle Access\rangle$ \\
& $\mid$ & \texttt{decision} $\langle Id\rangle$ \texttt{guard} $\langle Pred\rangle$ \\
& & $\langle Access\rangle$ \\
& $\mid$ & \texttt{exit} $\langle Id\rangle$ \texttt{output} $\langle Slot\rangle^{*}$ \\
$\langle Access\rangle$ & $::=$ & \texttt{read} $\langle Slot\rangle^{*}$ \texttt{write} $\langle Slot\rangle^{*}$ \\
$\langle Edge\rangle$ & $::=$ & $\langle Id\rangle \rightarrow \langle Id\rangle$ $[\,$\texttt{when} $\langle Guard\rangle\,]$
\end{tabular}
\vspace{4pt}
\hrule
\endgroup
\caption{Abstract syntax of WAST and WCFG.}
\label{fig:wast-wcfg-syntax}
\end{figure}

\begin{figure*}[t]
\centering
\includegraphics[width=\textwidth]{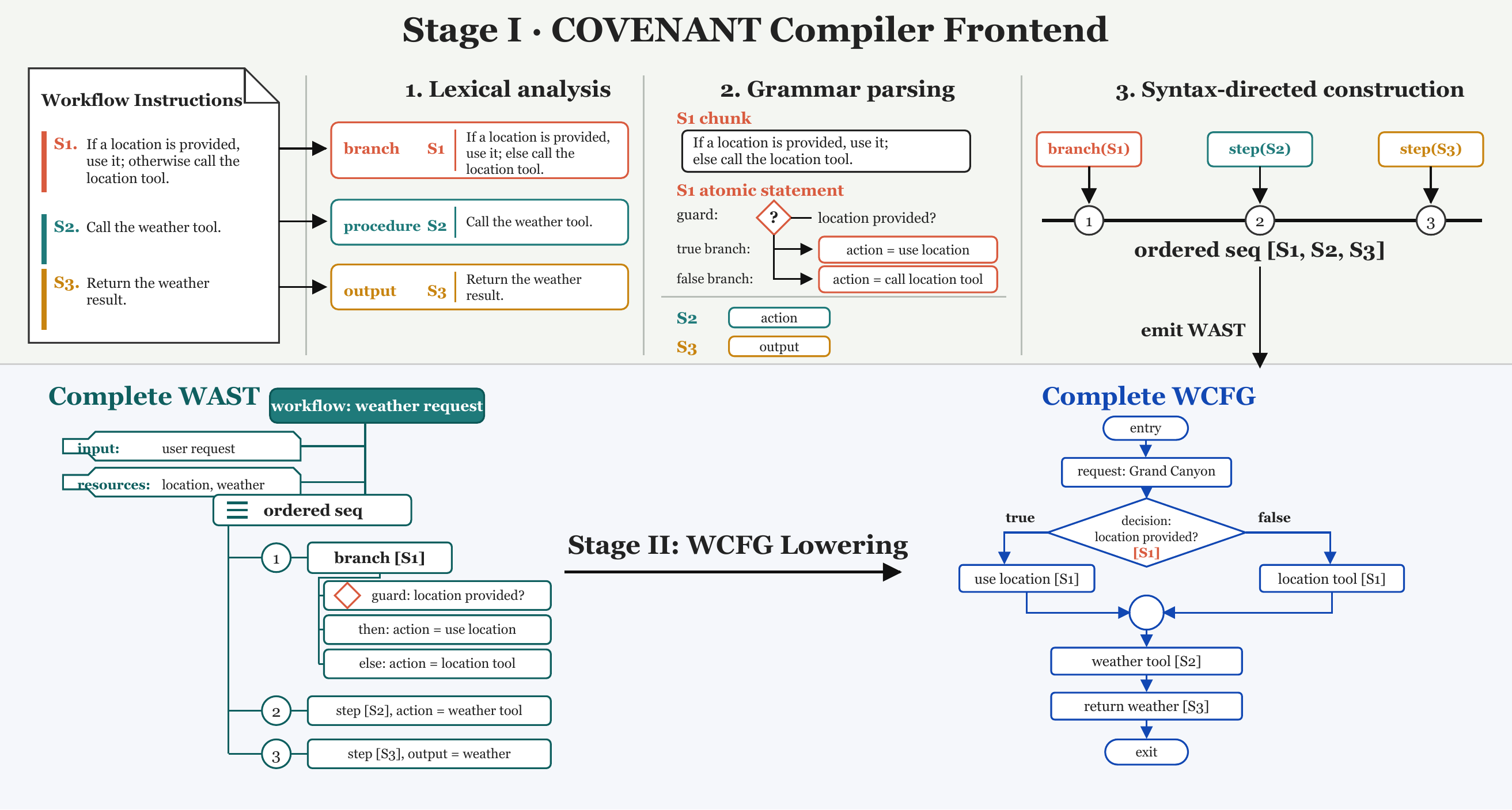}
\caption{Compiler-style construction and lowering of the Grand Canyon weather workflow from Figure~\ref{fig:motivating}. Stage I segments the instructions into labeled chunks, parses typed atomic statements, and assembles the WAST; Stage II lowers the WAST into a WCFG with explicit runtime transitions.}
\label{fig:compilation-example}
\end{figure*}

\subsubsection{Stage I: Hierarchical WAST Construction}
\label{subsec:wast-construction}

Conventional compiler frontends construct an abstract syntax tree (AST) through lexical analysis, grammar parsing, and syntax-directed construction~\cite{aho2006compilers}. COVENANT retains this three-pass organization but adapts it to instructions without fixed token classes or a complete grammar; Figure~\ref{fig:compilation-example} shows the intermediate artifacts. \textbf{Lexical analysis} uses layout rules to split the document at headings, numbered steps, and paragraph boundaries, then a schema-constrained LLM labels each chunk by procedural role. \textbf{Grammar parsing} treats Figure~\ref{fig:wast-wcfg-syntax} as an output grammar and extracts ordered actions, guards, and their explicit alternatives from each retained chunk. \textbf{Syntax-directed construction} deterministically maps actions to \texttt{step} nodes, guarded alternatives to \texttt{branch} nodes, and source order to \texttt{seq}, leaving unresolved associations explicit. The LLM therefore supplies local semantic interpretation, while the compiler controls global tree assembly.

\subsubsection{Stage II: WCFG Lowering}
\label{subsec:wcfg-construction}

The compiler's \emph{lowering stage} converts the tree into a control-flow graph. In Figure~\ref{fig:compilation-example}, the S1 branch becomes a location decision followed by two action blocks; both paths join before the S2 weather-tool block, which is followed by the S3 return block. In general, each ordered action maps to a \texttt{block}, a guard becomes a \texttt{decision} with true and false successors, and a \texttt{merge} reconnects the two branches before their next shared statement. Required final outputs become fields of an \texttt{exit} node. Algorithm~\ref{alg:wast-to-wcfg} in Appendix~\ref{app:wcfg-lowering} gives the complete recursive construction.

This lowering is path preserving: every entry-to-exit path in the WCFG corresponds to a path in the WAST, and every WAST path is preserved in the WCFG. Consequently, if the WAST contains only control paths permitted by the workflow instructions, the generated WCFG also contains no invalid control path, thereby aligning its control paths with the workflow-permitted traces in $\mathcal{T}(W,I)$ from Definition~\ref{def:workflow-alignment}. Appendix~\ref{app:wcfg-path-soundness} formalizes and proves these claims.

\section{Runtime Workflow Execution}
\label{sec:runtime-workflow-execution}
\label{sec:runtime-execution}

At runtime, workflow alignment reduces to two questions: does execution follow a WCFG path, and is each visited node executed correctly? Acting as an interpreter, the controller addresses the first by owning graph traversal and the second by checking node-bound instruction requirements before accepting a result. Failed checks trigger diagnostic feedback and repair without advancing the graph.

\subsection{Controller-Guided WCFG Execution}
\label{subsec:runtime-controller}

The controller owns the active-node cursor and exposes only that node's task, relevant interaction history $h_t$, supporting instructions, and response schema to the target agent. The agent still follows policy $\pi$, but this bounded context confines its current execution to the prescribed action. For an environment action, the proposal $a'_t$ and returned observation $o_{t+1}$ extend $h_t$ according to Equation~\ref{eq:agent-environment-step}; the structured result may also contain a state update, evidence, and branch choice. The controller advances only after acceptance, so rejection leaves the cursor and committed state unchanged. Appendix~\ref{app:algorithm} details the controller algorithm with a running example.

\noindent\textbf{Runtime Verification and Repair.} During offline compilation, COVENANT binds instruction constraints and evidence or effect obligations to WCFG nodes as \emph{semantic properties}. Protocol checks validate the controller--agent exchange, while property checks validate the proposed action, state update, or output. Failure retains the current node and supplies repair feedback; success commits the result before traversal advances, and exhausted retries block execution. Controller-owned traversal therefore targets trace misalignment, while the verify--repair--commit boundary targets execution misalignment. Appendix~\ref{app:semantic-properties} lists the properties and their enforcement guarantees.

\section{Evaluation}
\label{sec:evaluation}

The evaluation follows the two-part argument developed above: first establish whether the complete system improves task correctness and workflow alignment, then examine where the difference comes from and what it costs. We organize this analysis around four research questions (RQs). \textbf{RQ1} asks whether COVENANT improves end-to-end task correctness and reduces workflow-misalignment failures for the same agent interface and base model. \textbf{RQ2} asks in which evaluation scenarios the gains and regressions occur and whether they arise from trace or execution misalignment. \textbf{RQ3} studies the compiler, controller, and checking mechanisms through diagnostic variants. \textbf{RQ4} measures execution cost. The primary correctness metric is \emph{benchmark-native task success}: a binary pass or fail assigned by each benchmark's official evaluator.

\subsection{Experimental Setup}

Our evaluation uses 120 cases from GuideBench, ToolSandbox, and tau-bench, organized into seven workflow scenarios spanning conditional rules, argument canonicalization, insufficient information, multiple user turns, state dependencies, airline, and retail~\cite{guidebench2025,lu2024toolsandbox,yao2024taubench}. We retain each benchmark's official evaluator; Appendix~\ref{app:evaluation-details} reports the scenario composition and integration protocols.

We evaluate five agent interfaces---plain prompting, ReAct~\cite{yao2023react}, Claude Code, OpenClaw, and Hermes Agent---with five base models: GLM-5.1, GLM-5.2, MiniMax-M2.7, Kimi-K2.7-Code, and DeepSeek-V4-Pro. Each of the resulting 25 paired cells runs the same cases with the target agent alone and with COVENANT controlling that agent. The base model is fixed within each pair, so the comparison measures system effectiveness under a fixed model rather than equal compute. Appendix~\ref{app:evaluation-details} details execution modes, failure handling, and frozen artifacts.

\subsection{RQ1: Overall Effectiveness}

We answer RQ1 from two complementary perspectives. Table~\ref{tab:main-comparison} measures whether COVENANT improves end-to-end task correctness, while Figure~\ref{fig:misalignment-rate} measures whether it reduces failures that the saved execution evidence attributes to workflow misalignment.

\begin{table*}[t]
\centering
\small
\setlength{\tabcolsep}{5pt}
\newcommand{\pairscore}[2]{#1\,/\,\textbf{#2}}
\begin{tabular}{@{}lcccccc@{}}
\toprule
\textbf{Agent Interface} & \multicolumn{5}{c}{\textbf{Base Model}} & \textbf{Mean $\Delta$/Gap} \\
\cmidrule(lr){2-6}
& GLM-5.1 & GLM-5.2 & MiniMax & Kimi & DeepSeek & \\
\midrule
\multicolumn{7}{l}{\textit{Paired evaluation:}\quad Target Agent\,/\,\textbf{COVENANT + Target Agent}} \\
Prompting & \pairscore{44.17}{77.50} & \pairscore{47.50}{80.00} & \pairscore{36.67}{69.17} & \pairscore{41.67}{80.00} & \pairscore{40.83}{68.33} & \textbf{+32.83} \\
ReAct & \pairscore{44.17}{79.17} & \pairscore{46.67}{80.83} & \pairscore{42.50}{72.50} & \pairscore{45.83}{75.83} & \pairscore{36.67}{68.33} & \textbf{+32.17} \\
Claude Code & \pairscore{49.17}{75.83} & \pairscore{48.33}{82.50} & \pairscore{36.67}{65.83} & \pairscore{51.67}{74.17} & \pairscore{40.83}{65.83} & \textbf{+27.50} \\
OpenClaw & \pairscore{45.83}{82.50} & \pairscore{46.67}{82.50} & \pairscore{32.50}{61.67} & \pairscore{47.50}{77.50} & \pairscore{40.83}{68.33} & \textbf{+31.83} \\
Hermes & \pairscore{50.00}{83.33} & \pairscore{47.50}{82.50} & \pairscore{33.33}{54.17} & \pairscore{49.17}{59.17} & \pairscore{33.33}{60.00} & \textbf{+25.17} \\
\midrule
\multicolumn{7}{l}{\textit{Single-only workflow systems:}\quad Target Agent only; reference = best COVENANT (83.33)} \\
FlowAgent & 32.50 & 24.17 & 20.00 & 33.33 & 20.00 & \textbf{+57.33} \\
NLAH/LinguaClaw & 46.67 & 23.33 & 19.17 & 25.00 & 20.83 & \textbf{+56.33} \\
MermaidFlow & 25.00 & 27.50 & 23.33 & 24.17 & 23.33 & \textbf{+58.66} \\
Agent SOP & 43.33 & 50.00 & 47.50 & 28.33 & 45.83 & \textbf{+40.33} \\
\bottomrule
\end{tabular}
\caption{Official benchmark task success (\%) on 120 cases from three existing benchmarks. Paired rows report mean paired $\Delta$. For single-only systems, Gap is the best COVENANT cell (83.33\%) minus the baseline's five-model mean; it is not a paired effect.}
\label{tab:main-comparison}
\end{table*}

\begin{figure*}[t]
\centering
\includegraphics[width=0.94\textwidth]{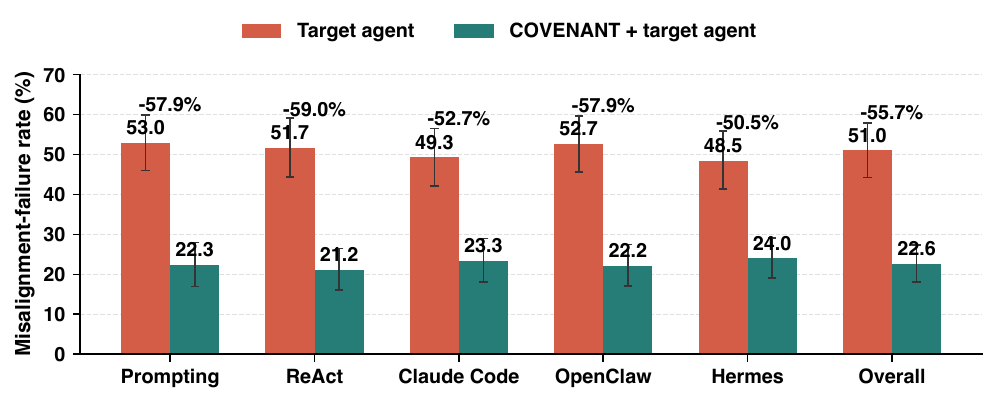}
\caption{Workflow-misalignment failure rates aggregated across five base models. GLM-5.1 is used only to annotate failed traces. Error bars are 95\% bootstrap intervals clustered by workflow case; percentages above each group are relative reductions.}
\label{fig:misalignment-rate}
\end{figure*}

\begin{figure*}[t]
\centering
\includegraphics[width=0.94\textwidth]{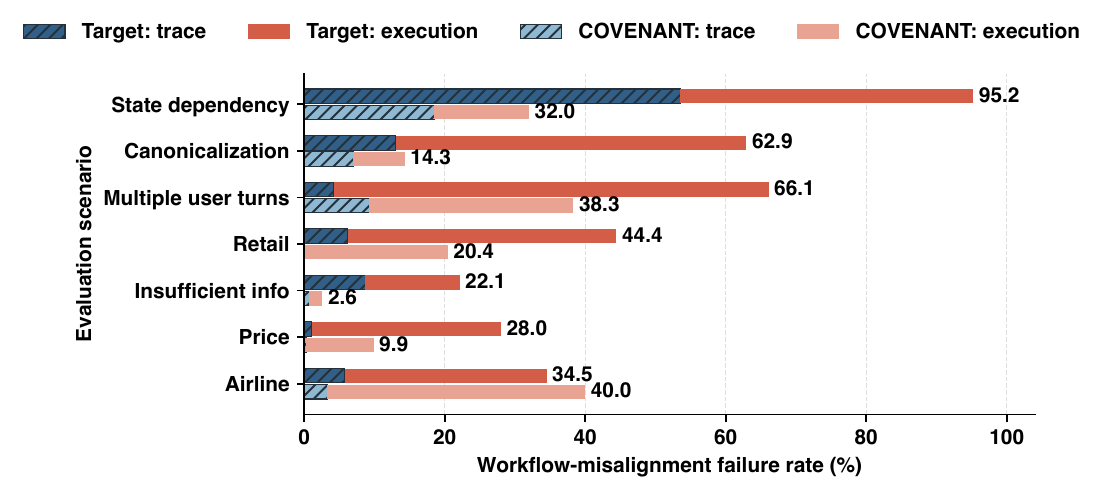}
\caption{Scenario-level GLM-5.1-attributed workflow-misalignment failure rates, decomposed into trace and execution misalignment.}
\label{fig:rq2-details}
\end{figure*}

COVENANT improves all 25 paired cells. The strongest configuration, Hermes Agent with GLM-5.1, solves 100/120 cases (83.33\%), versus 60/120 (50.00\%) without COVENANT. Across 3,000 paired instances, success rises from 1,296 (43.20\%) to 2,193 (73.10\%), a gain of 29.90 points. Because the same cases recur across cells, this aggregate is descriptive rather than a set of independent observations.

The lower panel places these paired results beside FlowAgent, NLAH/LinguaClaw, MermaidFlow, and Agent~SOP~\cite{flowagent2025,nlah2026,zheng2025mermaidflow,strandsagentsops2026}. Their five-model means range from 24.67\% to 43.00\%, leaving 40.33--58.66 point gaps from the best COVENANT result. Although these are unpaired reference gaps, they show the value of combining free-form workflow compilation with controller-owned traversal and node-level checking, rather than only generating or executing a structured workflow. Appendix~\ref{app:evaluation-details} documents unsupported cases and adapter scope.

To determine whether the corrected outcomes reflect alignment, GLM-5.1 labels each native failure as trace misalignment, execution misalignment, or unsupported attribution. The \emph{workflow-misalignment failure rate} is the fraction of all executions receiving either misalignment label; Appendix~\ref{app:evaluation-details} gives the acceptance and bootstrap protocols. Across 3,000 executions per mode, this rate falls from 51.03\% (1,531) to 22.60\% (678), a 55.72\% relative reduction. The clustered-bootstrap 95\% interval for the 28.43-point absolute reduction is 21.73--35.27 points. Figure~\ref{fig:misalignment-rate} shows reductions for every agent family and all 25 agent--model cells.

\subsection{RQ2: Scenario-Level Results and Misalignment Analysis}

Following the trace/execution distinction in Appendix~\ref{app:workflow-misalignment}, RQ2 locates the aggregate gain across scenarios and failure types. Native task success improves most in state dependency (+65.75 points), canonicalization (+46.95), and retail (+38.45); airline is the only negative scenario ($-4.00$). Figure~\ref{fig:rq2-details} reports each misalignment class over all executions in a scenario, aggregated across the 25 agent--model cells.

COVENANT lowers the attributed misalignment rate in six of seven scenarios. The largest reductions occur in state dependency (95.25\% to 32.00\%) and canonicalization (62.95\% to 14.32\%), where the task-success gains are also largest. Across all scenarios, the trace-misalignment rate falls from 12.93\% to 5.60\%, a 56.70\% relative reduction, while the execution-misalignment rate falls from 38.10\% to 17.00\%, a 55.38\% relative reduction. The decomposition also identifies a remaining traversal weakness: in multiple-user-turn tasks, execution misalignment decreases substantially but trace misalignment rises from 4.21\% to 9.26\%.

Airline is the only negative scenario. Although trace misalignment falls from 5.75\% to 3.25\%, execution misalignment rises from 28.75\% to 36.75\%; correspondingly, 73 paired instances succeed only for the target agent, whereas 57 succeed only with COVENANT. The regression is concentrated: five tasks account for 62 of the 73 target-only successes (84.93\%), and \texttt{airline\_0029} regresses in 24 of 25 agent--model cells. This task requires cancelling every qualifying single-passenger reservation, but the saved traces frequently show COVENANT cancelling only one reservation, leaving another obligation unresolved, or mishandling mandatory confirmation. These cases expose a limitation of the current mediated Airline integration: its workflow does not explicitly represent repeated multi-reservation closure or policy-dependent refusal and transfer conditions. It therefore reduces path errors but can stop after partial completion or continue under an inapplicable policy, increasing errors within selected steps.

\subsection{RQ3: Diagnostic Mechanism Analysis}

COVENANT's offline audit accepts all 120 generated WCFGs and rejects all 600 deterministically corrupted graphs, showing that malformed artifacts fail the checked invariants before runtime. Table~\ref{tab:ablation} then compares four runtime configurations using Claude Code with GLM-5.1 on 85 cases. \emph{Controller + FlowAgent PDL} exposes actions whose API preconditions hold but lets the agent select among them~\cite{flowagent2025}; \emph{WCFG + controller} owns traversal but disables property checks and repair; \emph{Full COVENANT} enables both.

\begin{table}[t]
\centering
\small
\setlength{\tabcolsep}{3.5pt}
\begin{tabular}{@{}p{0.43\columnwidth}rr@{}}
\toprule
Variant & Success & $\Delta$ from prior \\
\midrule
Base Claude Code & 34/85 (40.00\%) & reference \\
Controller + FlowAgent PDL & 36/85 (42.35\%) & +2.35 \\
WCFG + controller & 60/85 (70.59\%) & +28.24 \\
\shortstack[l]{Full COVENANT\\(WCFG + controller\\+ property checking)} & 63/85 (74.12\%) & +3.53 \\
\bottomrule
\end{tabular}
\caption{Runtime representation and control analysis with Claude Code and GLM-5.1.}
\label{tab:ablation}
\end{table}

Dependency-only control raises success from 40.00\% to 42.35\%; controller-owned WCFG traversal raises it to 70.59\%, and property checking with repair further raises it to 74.12\%. The largest gain therefore comes from explicit traversal, with a further 3.53-point checking gain. The FlowAgent-PDL row is a controlled hybrid rather than the official result in Table~\ref{tab:main-comparison}; because the variants span two implementation versions, we treat them as design evidence rather than causal estimates.

\subsection{RQ4: Cost}

Figure~\ref{fig:cost-overhead} reports normalized overhead over the same 3,000 paired instances.

\begin{figure}[t]
\centering
\includegraphics[width=\columnwidth]{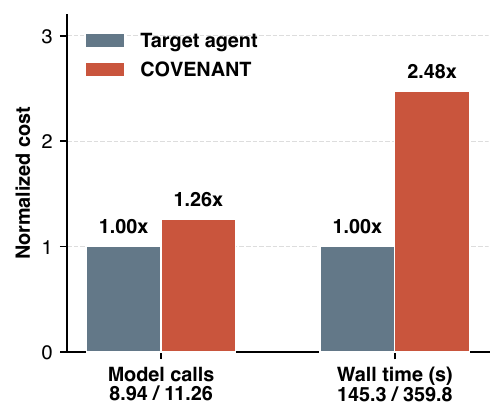}
\caption{Normalized model-call and wall-clock costs over 3,000 paired case instances. Values below each metric list Target agent / COVENANT means per case.}
\label{fig:cost-overhead}
\end{figure}

COVENANT raises model calls from 8.94 to 11.26 per case (1.26$\times$) and wall time from 145.31 to 359.84 seconds (2.48$\times$): controlled execution adds modest model calls but greater orchestration latency.

\noindent\textbf{Limitations and threats to validity.}
COVENANT does not yet guarantee faithful natural-language compilation, and its semantic verifiers may err. Appendix~\ref{app:limitations} details these boundaries and the scope of the controlled tests.

\section{Related Work}
\label{sec:related-work}

\paragraph{Workflow adherence and formal control.}
Instruction-following benchmarks evaluate response-level compliance, while GuideBench, ToolSandbox, and tau-bench extend evaluation to tool use and stateful interaction~\cite{ifeval2023,followbench2024,guidebench2025,lu2024toolsandbox,yao2024taubench}. SOPBench and JourneyBench further evaluate action permissibility, procedure completeness, or agreement with expected tool traces~\cite{sopbench2025,journeybench2026}. Agent-C and AgentLTL enforce temporal constraints over tool traces, but assume executable specifications that are manually encoded or synthesized from hand-written templates~\cite{agentc2025,agentltl2026}. COVENANT instead compiles free-form natural-language workflow instructions and defines alignment over the complete agent--environment trace.

\paragraph{Semantic parsing and process extraction.}
Semantic parsers and text-to-process systems map natural language into grammar-constrained structures or recover activities, branches, and control flow from procedural prose~\cite{kate2005learning,yin-neubig-2017-syntactic,friedrich2011process,bellan2023pet}. Process-conformance methods then compare observed traces with formal process models~\cite{rozinat2008conformance,vanderaa2018checking}. COVENANT adapts this separation between semantic interpretation and deterministic representation to LLM-agent execution, incorporating both environment observations and node-level realization into alignment.

\paragraph{Workflow generation and agent control.}
Routine, MermaidFlow, and Agent~SOP structure or generate agent workflows~\cite{routine2025,zheng2025mermaidflow,strandsagentsops2026}. FlowAgent supervises a procedure-description language with pre- and post-decision controllers, while RunAgent interprets plans written with reserved control constructs~\cite{flowagent2025,runagent2026}. Unlike these structured-input approaches, COVENANT begins with naturally authored, free-form instructions and uses its WCFG controller to own both traversal and committed state.

\section{Conclusion}
\label{sec:conclusion}

COVENANT treats natural-language workflow following as a complete-trace alignment problem and adapts compilation to agent execution. Its offline compiler transforms free-form instructions into a WAST and WCFG, while its runtime controller owns graph traversal and verifies each node before advancing. Across seven scenarios, the strongest configuration raises benchmark success from 50.00\% to 83.33\% and reduces attributed workflow-misalignment failures from 42.50\% to 15.83\%. These results demonstrate the value of explicit compilation and runtime control.

\clearpage
\appendix
% Appendix content for the combined arXiv preprint. The preamble, title,
% bibliography, and document boundaries are owned by main.tex.
\section{WAST-to-WCFG Construction}
\label{app:wcfg-lowering}

Algorithm~\ref{alg:wast-to-wcfg} gives the deterministic lowering procedure. \textsc{Lower} returns a subtree entry and its fall-through nodes, while \textsc{Link} connects them to the surrounding sequence or branch merge. \textsc{Task} and \textsc{AttachStateAccess} supply executable operations and state metadata.

\begin{algorithm}[H]
\caption{Deterministic lowering from a WAST to a WCFG}
\label{alg:wast-to-wcfg}
\footnotesize
\begin{algorithmic}[1]
\REQUIRE normalized WAST $A$ with root tree $T$
\ENSURE WCFG $G$
\STATE $G\gets\langle N=\emptyset,E=\emptyset,\Sigma(A)\rangle$
\STATE $s\gets\textsc{NewEntry}(G)$; $x\gets\textsc{NewExit}(G,\textsc{OutputSlots}(A))$
\STATE $(h,F)\gets\textsc{Lower}(T,G)$
\STATE $\textsc{AddEdge}(G,s,h)$
\STATE $\textsc{Link}(G,F,x)$
\STATE $\textsc{AttachStateAccess}(G,A)$; \textbf{return} $G$
\STATE \textbf{function} $\textsc{Lower}(t,G)$
\IF{$t=\texttt{seq}(t_1,\ldots,t_k)$}
  \STATE $(h,F)\gets\textsc{Lower}(t_1,G)$
  \FOR{$i=2$ to $k$}
    \STATE $(h_i,F_i)\gets\textsc{Lower}(t_i,G)$
    \STATE $\textsc{Link}(G,F,h_i)$
    \STATE $F\gets F_i$
  \ENDFOR
  \RETURN $(h,F)$
\ELSIF{$t=\texttt{step}(e_1,\ldots,e_m)$}
  \STATE $v\gets\textsc{NewBlock}(G,\textsc{Task}(t))$
  \RETURN $(v,\{v\})$
\ELSE
  \STATE $(g,t_{\mathsf{T}},t_{\mathsf{F}}?)\gets\textsc{Fields}(t)$
  \STATE $d\gets\textsc{NewDecision}(G,g)$; $m\gets\textsc{NewMerge}(G)$
  \STATE $(h_{\mathsf{T}},F_{\mathsf{T}})\gets\textsc{Lower}(t_{\mathsf{T}},G)$
  \STATE $\textsc{AddEdge}(G,d,h_{\mathsf{T}},\mathsf{true})$; $\textsc{Link}(G,F_{\mathsf{T}},m)$
  \IF{$t_{\mathsf{F}}$ exists}
    \STATE $(h_{\mathsf{F}},F_{\mathsf{F}})\gets\textsc{Lower}(t_{\mathsf{F}},G)$
    \STATE $\textsc{AddEdge}(G,d,h_{\mathsf{F}},\mathsf{false})$; $\textsc{Link}(G,F_{\mathsf{F}},m)$
  \ELSE
    \STATE $\textsc{AddEdge}(G,d,m,\mathsf{false})$
  \ENDIF
  \RETURN $(d,\{m\})$
\ENDIF
\end{algorithmic}
\end{algorithm}

\section{WCFG Path Soundness}
\label{app:wcfg-path-soundness}

Definition~\ref{def:workflow-alignment} defines $\mathcal{T}(W,I)$ as a set of complete action--observation traces, while a path through a WCFG is a sequence of graph nodes and branch edges. These objects therefore cannot be compared by direct set inclusion. Let $\operatorname{ctrl}_W(\tau)$ denote the \emph{control projection} of an interaction trace $\tau$: it retains the workflow-action labels and branch outcomes induced by $W$, while erasing environment observations and action arguments. For a set of traces $\mathcal{X}$, let $\operatorname{ctrl}_W(\mathcal{X})=\{\operatorname{ctrl}_W(\tau)\mid\tau\in\mathcal{X}\}$.

\noindent\begin{minipage}{\columnwidth}
\paragraph{WAST path language.}
Let $\mathcal{L}(t)$ be the control-path language of a WAST subtree $t$, and write $s=\texttt{step}(e_1,\ldots,e_m)$. The WAST grammar induces the following path language:
\begin{align*}
\mathcal{L}(s)
  &=\{\langle\textsc{Task}(s)\rangle\},\\
\mathcal{L}(\texttt{seq}(t_1,\ldots,t_k))
  &=\mathcal{L}(t_1)\cdots\mathcal{L}(t_k),\\
\mathcal{L}(\texttt{branch}(g,t_{\mathsf T},t_{\mathsf F}))
  &=\{\langle(g,\mathsf T)\rangle\}\mathcal{L}(t_{\mathsf T})\\
  &\quad\cup\{\langle(g,\mathsf F)\rangle\}\mathcal{L}(t_{\mathsf F}).
\end{align*}
\end{minipage}\par
Here juxtaposition denotes language concatenation; when the false subtree is absent, $\mathcal{L}(t_{\mathsf F})=\{\epsilon\}$. Let $\mathcal{L}_I(A_W)$ be the subset feasible for input $I$, meaning that every recorded branch outcome agrees with its guard under the action--observation prefix reaching that branch.

\paragraph{WCFG path language.}
For a WCFG $G_W$, let $\mathcal{P}_I(G_W)$ contain the projections of all feasible entry-to-exit paths. Projection erases \texttt{entry}, \texttt{merge}, and \texttt{exit} nodes, maps each \texttt{block} to its task label, and maps each outgoing \texttt{decision} edge to the corresponding $(g,\mathsf T)$ or $(g,\mathsf F)$ outcome. Feasibility uses the same guard evaluation as $\mathcal{L}_I(A_W)$.

\paragraph{Frontend soundness premise.}
Let $A_W$ be the WAST extracted from $W$. The natural-language frontend is \emph{path-sound} for input $I$ when
\begin{equation}
\mathcal{L}_I(A_W)
\subseteq
\operatorname{ctrl}_W\!\left(\mathcal{T}(W,I)\right).
\label{eq:frontend-path-soundness}
\end{equation}
This premise says that extraction may omit a permitted alternative, but it must not invent an action order or branch outcome that the workflow forbids. It is a semantic property of the WAST construction, not a consequence of graph-shape validation.

\begin{theorem}[Lowering Path Preservation]
\label{thm:lowering-path-preservation}
For any normalized WAST $A_W$ and input $I$, Algorithm~\ref{alg:wast-to-wcfg} produces a WCFG $G_W$ such that
\begin{equation}
\mathcal{P}_I(G_W)=\mathcal{L}_I(A_W).
\label{eq:lowering-path-preservation}
\end{equation}
\end{theorem}
\emph{Proof.}~We prove the stronger recursive invariant that, whenever $\textsc{Lower}(t,G)=(h,F)$, the projected paths from $h$ to the fall-through set $F$ are exactly $\mathcal{L}(t)$. For a \texttt{step}, the algorithm creates one block labeled by $\textsc{Task}(t)$, so the invariant holds directly. For a \texttt{seq}, the algorithm links every fall-through node of $t_i$ only to the entry of $t_{i+1}$. By the induction hypothesis, its paths are therefore exactly the concatenation $\mathcal{L}(t_1)\cdots\mathcal{L}(t_k)$. For a \texttt{branch}, the new decision has exactly one true edge and one false edge. Each edge enters the corresponding recursively lowered subtree, and both subtrees reconnect only at the new merge node; when the false subtree is absent, the false edge goes directly to that merge. Erasing the decision and merge structure while retaining the selected edge yields exactly the two alternatives in $\mathcal{L}(t)$. Finally, adding the global entry and exit nodes and attaching state access do not change the projected language. Thus $\mathcal{P}(G_W)=\mathcal{L}(A_W)$; filtering both sides by the same guards gives Equation~\ref{eq:lowering-path-preservation}.~$\square$

\begin{theorem}[WCFG Path Soundness]
\label{thm:wcfg-path-soundness}
If the frontend satisfies Equation~\ref{eq:frontend-path-soundness}, then every feasible path of the generated WCFG is permitted by the source workflow at the control-flow level:
\begin{equation}
\mathcal{P}_I(G_W)
\subseteq
\operatorname{ctrl}_W\!\left(\mathcal{T}(W,I)\right).
\label{eq:wcfg-path-soundness}
\end{equation}
\end{theorem}
\emph{Proof.}~Substitute Equation~\ref{eq:lowering-path-preservation} into the frontend-soundness premise in Equation~\ref{eq:frontend-path-soundness}.~$\square$

Theorem~\ref{thm:wcfg-path-soundness} establishes path-level alignment, but not yet the full membership condition in Definition~\ref{def:workflow-alignment}. A controller that advances only along feasible WCFG edges has a visited-node path in $\mathcal{P}_I(G_W)$. To further conclude that its complete realized trace $\tau_{\pi,\mathsf{Env}}(W,I)$ belongs to $\mathcal{T}(W,I)$, every visited block must also execute its represented action with conforming arguments, evidence, effects, and output. Runtime property checking targets this second condition; following the WCFG alone cannot establish it.

\section{Runtime Controller Algorithm}
\label{app:algorithm}

Algorithm~\ref{alg:runtime-controller} formalizes controller-owned traversal together with runtime verification and repair. It initializes the interaction history with $h_0=\langle I\rangle$ as in Section~\ref{sec:workflow-execution}. \textsc{RunNodeAgent} executes policy $\pi$ in the agent's environment and returns both a structured response and the resulting history; any action--observation pair already produced remains in that history even if the response is later rejected. \textsc{ParseAndProtocolChecks} applies the deterministic exchange checks, while \textsc{SemanticChecks} evaluates the enabled properties listed in Appendix~\ref{app:semantic-properties}. The merge notation $X\bowtie p$ means that only the state fields declared by patch $p$ are written into committed controller state $X$.

\begin{algorithm}[!t]
\caption{Controller-guided WCFG execution with verification and repair}
\label{alg:runtime-controller}
\footnotesize
\begin{algorithmic}[1]
\REQUIRE compiled $\mathsf{WCFG}_W$, workflow $W$, input $I$, policy $\pi$, environment $\mathsf{Env}$, retry budget $b$
\ENSURE realized interaction trace and output $y$, or \textsc{Blocked}
\STATE $h\gets\langle I\rangle$; $X\gets\emptyset$; $c\gets\textsc{Entry}(\mathsf{WCFG}_W)$
\WHILE{true}
  \IF{$c$ is an entry or merge node}
    \STATE $c\gets\textsc{UniqueSuccessor}(c,\mathsf{WCFG}_W)$
  \ELSE
    \STATE $j\gets 0$; $accepted\gets 0$; $feedback\gets\emptyset$
    \WHILE{$j\le b$ \textbf{and} $accepted=0$}
      \STATE $u\gets\textsc{BuildRequest}(c,X,h,feedback)$
      \STATE $(r,h)\gets\textsc{RunNodeAgent}(\pi,\mathsf{Env},W,u,h)$
      \STATE $(r,F)\gets\textsc{ParseAndProtocolChecks}(c,r,X)$
      \IF{$F=\emptyset$}
        \STATE $P\gets\textsc{Properties}(c)$
        \STATE $F\gets\textsc{SemanticChecks}(c,P,r,X,h)$
      \ENDIF
      \IF{$F=\emptyset$}
        \STATE $accepted\gets 1$
      \ELSE
        \STATE $\textsc{Discard}(r.\mathit{patch})$; $feedback\gets F$; $j\gets j+1$
      \ENDIF
    \ENDWHILE
    \IF{$accepted=0$}
      \RETURN $(h,\textsc{Blocked})$
    \ENDIF
    \STATE $X\gets X\bowtie r.\mathit{patch}$ \COMMENT{commit only after all checks pass}
    \IF{$c$ is an exit node}
      \STATE $y\gets r.\mathit{output}$
      \RETURN $(h,y)$
    \ENDIF
    \STATE $c\gets\textsc{SelectEdge}(c,r.\mathit{branch},\mathsf{WCFG}_W)$
  \ENDIF
\ENDWHILE
\end{algorithmic}
\end{algorithm}

\paragraph{Running example.}
Consider the Grand Canyon workflow in Figure~\ref{fig:compilation-example}. Initially, $h_0$ contains the user's request for the weather at the Grand Canyon. At the decision ``location provided?'', the controller exposes only this guard and its relevant context. Because $I$ already contains the location, policy $\pi$ selects the true branch, and the controller follows the encoded edge through the ``use location'' block and merge. At the subsequent weather-tool block, the structured response contains an action such as $a'_t=\textsc{WeatherTool}(\text{Grand Canyon})$, and the environment returns $o_{t+1}$. If the agent instead calls the location tool, skips directly to the final response, or reports a temperature unsupported by $o_{t+1}$, the controller keeps the cursor at the weather-tool block and returns the failed checks as repair feedback. Only an accepted result is committed before traversal advances to ``return weather.''

\FloatBarrier
\section{Semantic Properties and Controller Enforcement}
\label{app:semantic-properties}

Table~\ref{tab:semantic-properties} enumerates the semantic properties currently serialized in WCFG node contracts. A property states what an accepted node result must satisfy; protocol checks and semantic verifiers are the mechanisms that evaluate it. Properties are generated during compilation but evaluated against runtime responses and history. If the compiler cannot bind an extracted requirement to a node unambiguously, it retains the requirement as an audit record rather than enabling an unsound blocking check.

\begin{table}[!t]
\centering
\footnotesize
\setlength{\tabcolsep}{3pt}
\begin{tabular}{@{}p{0.31\columnwidth}p{0.63\columnwidth}@{}}
\toprule
Property ID & Scope, required condition, and runtime check \\
\midrule
\shortstack[l]{\texttt{source\_}\\\texttt{alignment}} & Executable nodes. Facts, values, choices, and state transitions must be grounded in the instruction span, accepted state, or observations. Checked through evidence records and the semantic verifier. \\
\shortstack[l]{\texttt{constraint\_}\\\texttt{satisfaction}} & Nodes with conditions or constraints. Applicable conditions, exceptions, and prohibitions must hold before accepting a decision or effect. Checked through guard/state rules and the semantic verifier. \\
\shortstack[l]{\texttt{candidate\_}\\\texttt{coverage}} & Derive, select, compute, and decide nodes. Every relevant candidate must be selected, rejected, deferred, or marked inapplicable with evidence. Checked through candidate/plan ledgers and the semantic verifier. \\
\shortstack[l]{\texttt{objective\_}\\\texttt{consistency}} & Nodes with an explicit objective. The selected result must satisfy that objective over feasible candidates and constraints. Checked through plan-ledger, numeric-consistency, and semantic checks. \\
\shortstack[l]{\texttt{output\_}\\\texttt{contract}} & Exit nodes. Required output fields and their workflow-defined meanings must be present. Checked through deterministic schema/required-field rules and the semantic verifier. \\
\shortstack[l]{\texttt{final\_trace\_}\\\texttt{consistency}} & Exit nodes. The final output must agree with the selected result, calculations, evidence, accepted state, and observed effects. Checked through ledger, numeric, effect-binding, and semantic checks. \\
\bottomrule
\end{tabular}
\caption{Semantic properties implemented by COVENANT.}
\label{tab:semantic-properties}
\end{table}

\paragraph{Controller enforcement.}
The properties above specify what an accepted node result must satisfy, while Algorithm~\ref{alg:runtime-controller} determines how their verdicts control execution. For every property enabled as a blocking check, the following guarantees establish that a violating proposal cannot change committed controller state or advance the WCFG.

\begin{theorem}[Core State Admission]
\label{thm:gate-committed}
Every state update committed by the reusable controller has passed every blocking check enabled for its proposal.
\end{theorem}
\emph{Proof.}~Algorithm~\ref{alg:runtime-controller} reaches its state-merge statement only when both protocol and enabled semantic-property checks return no blocking finding. Therefore every committed patch has passed every enabled blocking check.~$\square$

\begin{theorem}[Core Rejected-Patch Isolation]
\label{thm:gate-monotonicity}
If all attempts for a reusable-controller step are rejected, none of their state updates modifies committed controller state; a retry begins from the last accepted state.
\end{theorem}
\emph{Proof.}~The rejection branch discards the proposed patch before the state-merge statement, so a retry reads the last committed controller state.~$\square$

\paragraph{Scope.}
These guarantees cover committed controller state, not complete workflow semantics. They exclude audit-only properties, paths that bypass nodewise control, semantic errors in the compiled graph, and external environment state or actions. Because an external action may occur before post-execution checking, a rejected attempt can remain in the interaction history unless the environment exposes a pre-execution boundary.

\section{Detailed Evaluation Protocol}
\label{app:evaluation-details}

\paragraph{Workload composition.}
The 120 cases are grouped by workflow setting and official scoring procedure. GuideBench contributes 15 price-rule cases, which stress whether a rule is applied only under its stated conditions. ToolSandbox contributes argument canonicalization (19 cases), insufficient information (17), multiple user turns (19), and state dependency (16). Tau-bench contributes airline (16) and retail (18), both of which require every state change to follow benchmark policy. We retain GuideBench's label comparison, ToolSandbox's scenario evaluator, and tau-bench's environment reward.

\paragraph{Paired agents and models.}
An \emph{agent interface} is the prompting and tool-execution shell through which a base model receives and performs a task. The five interfaces are plain prompting, ReAct, Claude Code, OpenClaw, and Hermes Agent; the five models are GLM-5.1, GLM-5.2, MiniMax-M2.7, Kimi-K2.7-Code, and DeepSeek-V4-Pro. Each interface--model cell runs all 120 cases once with the target agent alone and once with COVENANT controlling it. COVENANT adds compiler, controller, and verifier calls, so the pair fixes the base model but not compute. We rerun model-service failures such as timeouts or unavailable endpoints and merge replacement runs by case identifier. All 25 cells have zero unresolved service failures and retain official evaluator outputs and execution traces.

\paragraph{Execution modes and artifacts.}
GuideBench uses fused execution, in which one structured model call receives the full compiled graph. ToolSandbox and tau-bench use mediated execution: an adapter derives a pending action from the task, transcript, and tool schema, checks the proposal, and passes an accepted action to the benchmark environment. These adapters include benchmark-specific argument normalization, validation, and proposal-repair rules. Table~\ref{tab:main-comparison} therefore evaluates the complete integrations rather than the reusable controller in isolation. COVENANT runs use fixed WAST and WCFG artifacts generated during development; the evaluation does not independently establish that compilation preserves every source requirement.

\paragraph{Workflow-misalignment annotation.}
For every execution that fails its native evaluator, GLM-5.1 receives the saved execution record and assigns trace misalignment, execution misalignment, or no supported attribution under Appendix~\ref{app:workflow-misalignment}. We accept a misalignment label only at confidence at least 0.8 and with no review flag; ambiguous failures remain in the denominator but not the numerator. Target-agent and COVENANT executions use the same prompt, taxonomy, acceptance rule, and denominator. Confidence intervals use 10,000 cluster-bootstrap resamples with seed 20260720, resampling the 120 workflow cases while retaining all agent--model repetitions associated with each sampled case.

\paragraph{Single-only workflow systems.}
The lower panel of Table~\ref{tab:main-comparison} provides unpaired context from systems that generate or execute structured workflows. MermaidFlow lacks a tau-bench action adapter for 34 of 120 cases per model, and several NLAH/LinguaClaw and MermaidFlow runs lack complete traces. Unsupported cases remain failures in the stated denominator; the reported gaps therefore contextualize system-level coverage and success rather than estimate a paired COVENANT effect.

\section{Limitations and Threats to Validity}
\label{app:limitations}

The natural-language frontend does not provide an end-to-end soundness guarantee that every WAST and WCFG faithfully captures its source instructions; semantic validation of this compilation step remains future work. Semantic verification is also selective: some properties are retained only for audit, and an LLM verifier can reject a correct result or accept an incorrect one. Theorem~\ref{thm:gate-committed} therefore establishes controller-state admission rather than end-to-end semantic correctness.

The ablation combines two saved implementation versions and shows a small checking gain on 85 cases but a small loss on all 120, so it provides design evidence rather than a clean causal estimate. The controlled stress tests isolate two structural pressures but do not represent all dependencies in production workflows.

\section{Per-Scenario Results and Misalignment Analysis}
\label{app:per-domain}

All 25 paired cells have zero unresolved model-service failures after case-level replacement. Across their 3,000 paired instances, both settings pass 1,112 and both fail 623; the remaining outcomes are 1,081 recoveries and 184 regressions under the definitions in RQ2. Table~\ref{tab:micro-deltas} reports the corresponding cell-level gains.

\begin{table}[!b]
\centering
\small
\setlength{\tabcolsep}{3.5pt}
\begin{tabular}{l rrrrr}
\toprule
Agent interface & \rotatebox{55}{GLM-5.1} & \rotatebox{55}{GLM-5.2} & \rotatebox{55}{MiniMax} & \rotatebox{55}{Kimi} & \rotatebox{55}{DeepSeek} \\
\midrule
Prompting & +33.33 & +32.50 & +32.50 & +38.33 & +27.50 \\
ReAct & +35.00 & +34.17 & +30.00 & +30.00 & +31.67 \\
Claude Code & +26.67 & +34.17 & +29.17 & +22.50 & +25.00 \\
OpenClaw & +36.67 & +35.83 & +29.17 & +30.00 & +27.50 \\
Hermes & +33.33 & +35.00 & +20.83 & +10.00 & +26.67 \\
\bottomrule
\end{tabular}
\caption{Task-success differences in percentage points: COVENANT minus the base interface.}
\label{tab:micro-deltas}
\end{table}

\section{Workflow Misalignment}
\label{app:workflow-misalignment}

Under the benign-environment assumption in Section~\ref{sec:workflow-execution}, the environment may faithfully execute a \emph{wrong} action, but it does not itself introduce the deviation. Inspection of policy failures reveals two recurring manifestations. In \emph{trace misalignment}, the policy drifts to the wrong procedural path by selecting an unsupported branch, skipping a required step, or changing the intended order. In Figure~\ref{fig:motivating}, the user already provides ``Grand Canyon,'' but the agent takes the branch that tries to resolve the location again. In \emph{execution misalignment}, the procedural path is appropriate but the policy executes a selected step incorrectly. In the airline example in Figure~\ref{fig:motivating}, the profile-tool step requires calling the profile tool to retrieve the passenger's date of birth, but the agent fabricates a value.

A further investigation shows that workflow alignment degrades as workflow instructions contain more branch points and as individual workflow steps involve more environment interactions. To examine these two factors in isolation, we conduct controlled experiments with a ReAct agent powered by GLM-5.1, as shown in Figure~\ref{fig:synthetic-scaling}. We report the \emph{alignment rate}: the fraction of cases whose complete execution satisfies the alignment requirement in Equation~\ref{eq:workflow-alignment}. In the branch-selection study, this requires the produced branch trace to exactly match the oracle path. As the number of successive branch points increases from 2 to 10, alignment falls from 80\% to 25\%, a decrease of 55 percentage points. In the tool-interaction study, every prescribed interaction and resulting value must satisfy the corresponding workflow step. Across the tested range, increasing the required interactions within one workflow step from 3 to 27 reduces alignment from 100\% to 66.7\%, a decrease of 33.3 percentage points, although the intermediate levels are not monotonic.

More branch points increase the number and depth of path decisions that the agent must preserve, creating more opportunities for trace misalignment. More environment interactions require the agent to issue more actions, track more feedback, and bind more returned values, creating more opportunities for execution misalignment. These patterns amplify well-known LLM limitations in complex tasks: long contexts and distractors weaken evidence retrieval~\cite{liu2024lost,shi2023distract}, while compositional tasks expose a gap between solving individual subproblems and their conjunction~\cite{press2024compositionality}.

\begin{figure}[!t]
\centering
\includegraphics[width=\columnwidth]{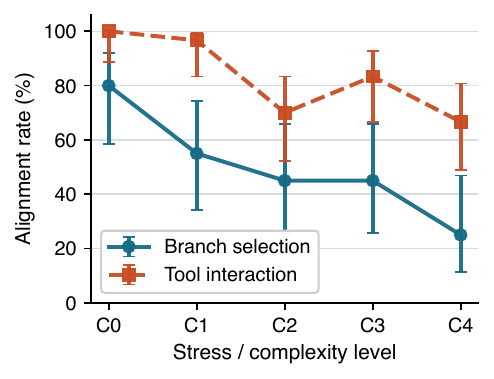}
\caption{Controlled stress tests of workflow alignment, using 20 branch-selection cases and 30 tool-interaction cases per level.}
\label{fig:synthetic-scaling}
\end{figure}

\bibliography{references}

\end{document}